\documentclass[graybox]{svmult}

\usepackage{mathptmx}       % selects Times Roman as basic font
\usepackage{helvet}         % selects Helvetica as sans-serif font
\usepackage{courier}        % selects Courier as typewriter font
\usepackage{type1cm}        % activate if the above 3 fonts are
\usepackage{makeidx}         % allows index generation
\usepackage{graphicx}        % standard LaTeX graphics tool
\usepackage{multicol}        % used for the two-column index
\usepackage[bottom]{footmisc}% places footnotes at page bottom

\usepackage{amsfonts}
\usepackage{amsmath}
\usepackage{amssymb}
\usepackage{algorithm}
\usepackage{algpseudocode}
\usepackage{bm}
\usepackage{verbatim}

\makeindex             % used for the subject index
\begin{document}

\title*{Nonlinear Eigenproblems in Data Analysis - Balanced Graph Cuts and the RatioDCA-Prox}
\titlerunning{Balanced Graph Cuts and the RatioDCA-Prox}
\author{Leonardo Jost, Simon Setzer and Matthias Hein}
\institute{Department of Mathematics and Computer Science, Saarland University, Saarbr{\"u}cken, Germany\\
\email{leo@santorin.cs.uni-sb.de, simon.setzer@gmail.com, hein@math.uni-sb.de}}
\maketitle

\abstract{It has been recently shown that a large class of balanced graph cuts allows for an exact relaxation into a nonlinear eigenproblem.
We review briefly some of these results and propose a family of algorithms to compute nonlinear eigenvectors which encompasses previous work as special cases.
We provide a detailed analysis of the properties and the convergence behavior of these algorithms and then discuss their application in the area of balanced graph cuts.
}


\section{Introduction}\label{Jost_Setzer_Hein:sec:intro}
\index{clustering,graph cut!balanced}

Spectral clustering is one of the standard methods for graph-based clustering \cite{Jost_Setzer_Hein:Lux2007}. It is based on the spectral relaxation of the so called normalized
cut, which is one of the most popular criteria for balanced graph cuts. While the spectral relaxation is known to be loose \cite{Jost_Setzer_Hein:GM98}, tighter relaxations based
on the graph $p$-Laplacian have been proposed in \cite{Jost_Setzer_Hein:BueHei2009}. Exact relaxations for the Cheeger cut based on the nonlinear eigenproblem of the graph $1$-Laplacian have been proposed in \cite{Jost_Setzer_Hein:SB10, Jost_Setzer_Hein:HeiBue2010}.
In \cite{Jost_Setzer_Hein:HeiSet2011} the general balanced graph cut problem of an undirected, weighted graph $(V,E)$ is considered. Let $n=|V|$ and denote the  
weight matrix of the graph by $W=(w_{ij})_{i,j=1}^n$, then the general balanced graph cut criterion can be written as
\begin{equation*}
 \mathop{\rm arg\,min}\limits_{A\subset V}\frac{\mathrm{cut}(A,\overline A)}{\hat S(A)},
\end{equation*}
where $\overline A=V\setminus A$, $\mathrm{cut}(A,\overline{A})=\sum_{i \in A, j \in \overline{A}} w_{ij}$, and $\hat S\colon 2^V\to\mathbb{R}_+$ is a symmetric and nonnegative balancing function. Exact relaxations of such balanced graph cuts and relations to corresponding nonlinear eigenproblems are discussed in \cite{Jost_Setzer_Hein:HeiSet2011} and are briefly reviewed in Section \ref{Jost_Setzer_Hein:sec:relaxation}.  A further generalization to hypergraphs has been established in \cite{Jost_Setzer_Hein:HeiSet2013}. 

There exist different approaches to minimize the exact continuous relaxations. However, in all cases the problem boils down to the minimization of a ratio of a convex and a difference of convex functions.
The two lines of work of \cite{Jost_Setzer_Hein:BLUB2012,Jost_Setzer_Hein:BLUB12} and \cite{Jost_Setzer_Hein:HeiBue2010,Jost_Setzer_Hein:HeiSet2011} have developed different algorithms for this problem, which have been compared in \cite{Jost_Setzer_Hein:BLUB2012}. We show that both types of algorithms are special cases of our new algorithm RatioDCA-prox introduced in Section \ref{Jost_Setzer_Hein:sec:algorithm}. We provide a unified analysis of the properties and the convergence behavior of RatioDCA-prox. Moreover, in Section \ref{Jost_Setzer_Hein:sec:algforcuts} we prove stronger convergence results when
the RatioDCA-prox is applied to the balanced graph cut problem or, more generally, problems where one minimizes nonnegative ratios of Lovasz extensions of set functions. Further, we discuss the choice of the relaxation of the balancing function in \cite{Jost_Setzer_Hein:HeiSet2011} and show that from a theoretical perspective the Lovasz extension is optimal which is supported by the numerical results in Section \ref{Jost_Setzer_Hein:sec:experiments}.

\section{Exact Relaxation of Balanced Graph Cuts}\label{Jost_Setzer_Hein:sec:relaxation}
\index{graph cut!balanced,Lovasz extension}
A key element for the exact continuous relaxation of balanced graph cuts is the Lovasz extension of a function on the power set $2^V$ to $\mathbb{R}^V$.% , which we will also inspect later in section \ref{Jost_Setzer_Hein:sec:algforcuts} is the Lovasz extension.
\begin{definition}\label{Jost_Setzer_Hein:def:Lovasz}
Let $\hat{S}:2^V \rightarrow \mathbb{R}$ be a set function with $\hat{S}(\emptyset)=0$.
Let $f \in \mathbb{R}^V$, let $V$ be ordered such that $f_1\leq f_2 \leq \ldots  \leq f_n$
and define $C_i = \{ j  \in V \, | \, j > i\}$. %and $C_0=V$.
Then, the Lovasz extension $S:\mathbb{R}^V \rightarrow \mathbb{R}$ of $\hat{S}$ is given by %the convex function $S: \mathbb{R}^V \rightarrow \mathbb{R}$, 
\begin{align*}
S(f) &=\, \sum_{i=1}^{n} f_i \Big(\hat{S}(C_{i-1}) - \hat{S}(C_i)\Big)=\, \sum_{i=1}^{n-1} \hat{S}(C_i)(f_{i+1}-f_i) + f_1 \hat{S}(V).
\end{align*}
Note that for the characteristic function of a set $C \subset V$, we have $S(\mathbf{1}_C)=\hat{S}(C)$.
\end{definition}
The Lovasz extension is convex if and only if $\hat{S}$ is submodular \cite{Jost_Setzer_Hein:Bac2013} and 
every Lovasz extension can be written as a difference of convex functions \cite{Jost_Setzer_Hein:HeiSet2011}. Moreover, the Lovasz extension of a symmetric set function
is positively one-homogeneous\footnote{A function $A\colon\mathbb{R}^n\to\mathbb{R}$ is (positively) $p$-homogeneous if $A(\nu x)=\nu^pA(x)$ for all $\nu\in\mathbb{R}$ ($\nu\ge0$).
In the following we will call functions just homogeneous when referring to positive homogeneity.} 
and preserves non-negativity, that is $S(f)\geq 0, \, \forall f \in \mathbb{R}^V$ if $\hat{S}(A)\geq 0, \, \forall A \subset V$.
It it well known, see e.g. \cite{Jost_Setzer_Hein:HeiSet2013}, that the Lovasz extension of the submodular cut function, $\hat{R}(A)=\mathrm{cut}(A,\overline{A})$, yields the total variation on a graph,
\begin{equation}\label{Jost_Setzer_Hein:eq:TV}
R(f) = \frac{1}{2}\sum_{i,j=1}^n w_{ij}|f_i-f_j|.
\end{equation}
Theorem \ref{Jost_Setzer_Hein:th:sets} shows exact continuous relaxations of balanced graph cuts \cite{Jost_Setzer_Hein:HeiSet2011}. A more general 
version for the class of constrained fractional set programs is given in \cite{Jost_Setzer_Hein:BueEtAl2013}.
\begin{theorem}\label{Jost_Setzer_Hein:th:sets}
\index{relaxation!exact}
Let $G=(V,E)$ be an undirected, weighted graph and $S:V \rightarrow \mathbb{R}$
and let $\hat{S}:2^V \rightarrow \mathbb{R}$ be symmetric  with $\hat{S}(\emptyset)=0$, then
\begin{align*} 
\mathop{\rm min}\nolimits\limits_{f \in \mathbb{R}^V} \frac{\frac{1}{2}\sum_{i,j=1}^n w_{ij}|f_i-f_j|}{S(f)} = \mathop{\rm min}\nolimits\limits_{A \subset V} \frac{\mathrm{cut}(A,\overline{A})}{\hat{S}(A)},
\end{align*}
if either one of the following two conditions holds 
\begin{enumerate}
\item $S$ is one-homogeneous, even, convex and $S(f+\alpha \mathbf{1})=S(f)$ for all $f \in \mathbb{R}^V$, $\alpha \in \mathbb{R}$ 
      and $\hat{S}$ is defined as $\hat{S}(A):=S(\mathbf{1}_A)$ for all $A \subset V$.
\item $S$ is the Lovasz extension of the non-negative, symmetric set function $\hat{S}$ with $\hat{S}(\emptyset)=0$.
\end{enumerate}
Let $f \in \mathbb{R}^V$ and denote by $C_t:=\{i \in V\,|\, f_i > t\}$, then it holds under both conditions,
\begin{align*}
 \mathop{\rm min}\nolimits\limits_{t \in \mathbb{R}} \frac{\mathrm{cut}(C_t,\overline{C_t})}{\hat{S}(C_t)} \leq \frac{\frac{1}{2}\sum_{i,j=1}^n w_{ij}|f_i-f_j|}{S(f)}.
\end{align*}
\end{theorem}
We observe that the exact continuous relaxation corresponds to a minimization problem of a ratio of non-negative, one-homogeneous functions, where the enumerator is convex and the denominator can be written as a difference of convex functions.

\section{Minimization of Ratios of Non-negative Differences of Convex Functions via the RatioDCA-prox} \label{Jost_Setzer_Hein:sec:minimization}
\index{difference of convex,minimization!nonconvex}
We consider in this paper continuous optimization problems of the form
\begin{equation}\label{Jost_Setzer_Hein:eq:minprob}
 \mathop{\rm min}\nolimits_{f \in \mathbb{R}^V} F(f), \quad \textrm{ where } \quad F(f)=\frac{R(f)}{S(f)}=\frac{R_1(f)-R_2(f)}{S_1(f)-S_2(f)},
\end{equation}
where $R_1,R_2,S_1,S_2$ are convex and one-homogeneous and $R(f)=R_1(f)-R_2(f)$ and $S(f)=S_1(f)-S_2(f)$ are non-negative. Thus we are minimizing a non-negative ratio of d.c.
(difference of convex) functions. As discussed above the exact continuous relaxation of Theorem \ref{Jost_Setzer_Hein:th:sets} leads exactly to such a problem, where $R_2(f)=0$ and 
$R_1(f)=\frac{1}{2}\sum_{i,j=1}^n w_{ij}|f_i-f_j|$. Different choices of balancing functions lead to different functions $S$.

While \cite{Jost_Setzer_Hein:HeiBue2010,Jost_Setzer_Hein:BLUB12, Jost_Setzer_Hein:BLUB2012} consider only algorithms for the minimization of ratios of convex functions, 
in \cite{Jost_Setzer_Hein:HeiSet2011} the RatioDCA has been proposed for the minimization of problems of type \eqref{Jost_Setzer_Hein:eq:minprob}. The generalized version RatioDCA-prox is a 
family of algorithms which contains the work of \cite{Jost_Setzer_Hein:HeiBue2010,Jost_Setzer_Hein:HeiSet2011,Jost_Setzer_Hein:BLUB12, Jost_Setzer_Hein:BLUB2012} as special cases and allows us 
to treat the minimization problem \eqref{Jost_Setzer_Hein:eq:minprob} in a unified manner.

\subsection{The RatioDCA-prox algorithm}
\label{Jost_Setzer_Hein:sec:algorithm}
\index{RatioDCA-prox,convex,clarke subdifferential}
The RatioDCA-prox algorithm for minimization of \eqref{Jost_Setzer_Hein:eq:minprob} is given in Algorithm \ref{Jost_Setzer_Hein:alg:RDP}. In each step one has to solve the convex optimization problem
\begin{equation}\label{Jost_Setzer_Hein:eq:innerprob}
 \mathop{\rm min}\nolimits_{G(u)\le1} \Phi^{c^k}_{f^k}(u),
\end{equation}
which we denote as the \emph{inner problem} in the following with
\begin{equation*}
\Phi_{f^k}^{c^k}(u):=R_1(u)-\left\langle{u,r_2(f^k)}\right\rangle+\lambda^k\Big(S_2(u)-\left\langle{u,s_1(f^k)}\right\rangle\Big)-c^k\left\langle{u,g(f^k)}\right\rangle
\end{equation*}
and $c^k\ge0$. As the constraint set we can choose any set containing a neighborhood of $0$, such that the inner problem is bounded from below,
i.e. any nonnegative convex $p$-homogeneous $(p\ge1)$ function $G$. Although a slightly more general formulation is possible,
we choose the constraint set to be compact, i.e. $G(f)=0\Leftrightarrow f=0$. Moreover, $s_1(f^k) \in \partial S_1(f^k)$, $r_2(f^k) \in \partial R_2(f^k)$, $g(f^k)\in\partial G(f^k)$, where $\partial S_1,\partial R_2,\partial G$ are the subdifferentials.
Note that for any $p$-homogeneous function $A$ we have the generalized Euler identity \cite[Theorem 2.1]{Jost_Setzer_Hein:YanWei08} that is %$\left\langle f^k,g(f^k)\right\rangle=p\,G(f^k)$.
$\left\langle{f,a(f)}\right\rangle=p\,A(f)$ for all $a(f)\in\partial A(f)$.

Clearly $\Phi_{f^k}^{c^k}$ is also one-homogeneous and with the Euler identity we get $\Phi_{f^k}^{c^k}(f^k)=-c^kpG(f^k)\le0$ so we can always find minimizers at the boundary.

\begin{algorithm}[!h]
   \caption{\label{Jost_Setzer_Hein:alg:RDP}{\bf RatioDCA-prox} -- Minimization of a ratio of non-negative, one-homogeneous d.c. functions}
\begin{algorithmic}[1]
   \State {\bfseries Initialization:} $f^0 = \text{random}$ with $G(f^0) = 1$, $\lambda^0 =F(f^0)$
   \Repeat 
   \State \text{find} $s_1(f^k) \in \partial S_1(f^k)$, $r_2(f^k) \in \partial R_2(f^k)$, $g(f^k)\in\partial G(f^k)$
   \State \text{find} $f^{k+1} \in \mathop{\rm arg\,min}\limits_{G(u) \leq 1} \Phi_{f^k}^{c^k}(u)$ 
   \State $\lambda^{k+1}= F(f^{k+1})$
	\Until $f^{k+1} \in \mathop{\rm arg\,min}\limits_{G(u) \leq 1} \Phi_{f^{k+1}}^{c^{k+1}}(u)$
\end{algorithmic}
\end{algorithm}

The difference to the RatioDCA in \cite{Jost_Setzer_Hein:HeiSet2011} is the additional proximal term $-c^k\left\langle u,g(f^k)\right\rangle$ in $\Phi^{c^k}_{f^k}(u)$ and the choice of $G$. 
It is interesting to note that this term can be derived by applying the RatioDCA to a different d.c. decomposition of $F$. Let us write $F$ as
\begin{equation}\label{Jost_Setzer_Hein:eq:dc-F}
F=\frac{R'_1-R'_2}{S'_1-S'_2}=\frac{(R_1+c_RG)-(R_2+c_RG)}{(S_1+c_SG)-(S_2+c_SG)}
\end{equation}
with arbitrary $c_R,c_S\ge0$. If we now define $c^k:=c_R+\lambda^k c_S$, 
the function to be minimized in the inner problem of the RatioDCA reads
\begin{equation*}
\begin{aligned}
\Phi'_{f^k}(u)=R'_1(u)-\left\langle u,r'_2(f^k)\right\rangle+\lambda^k\left(S'_2(u)-\left\langle u,s'_1(f^k)\right\rangle\right)
=\Phi^{c_k}_{f^k}(u)+c^kG(u),
\end{aligned}
\end{equation*}
which is not necessarily one-homogeneous anymore. The following lemma implies that
the minimizers of the inner problem of RatioDCA-prox and of RatioDCA applied to the d.c.-decomposition \eqref{Jost_Setzer_Hein:eq:dc-F} can be 
chosen to be the same.

\begin{lemma}
\label{Jost_Setzer_Hein:le:decomposition}
For $G(f^k)=1$ we have $\mathop{\rm arg\,min}\limits_{G(u)\le1}\Phi_{f^k}'(u)\supseteq\mathop{\rm arg\,min}\limits_{G(u)\le 1}\Phi^{c^k}_{f^k}(u)$.
Moreover, 
\begin{enumerate}
 \item if $p>1,c^k>0$ then $\mathop{\rm arg\,min}\limits_{u}\Phi_{f^k}'(u)\supseteq\nu\cdot\mathop{\rm arg\,min}\limits_{G(u)\le 1}\Phi^{c^k}_{f^k}(u)$ for some $\nu\ge1$,
 \item if $f^k\in\mathop{\rm arg\,min}\limits_{G(u)\le 1}\Phi^{c^k}_{f^k}(u)$ then $\mathop{\rm arg\,min}\limits_{u}\Phi_{f^k}'(u)\supseteq\mathop{\rm arg\,min}\limits_{G(u)\le 1}\Phi^{c^k}_{f^k}(u)$.
\end{enumerate}

\end{lemma}
\begin{proof}\smartqed
For fixed $\xi\ge0$ it follows from the one-homogeneity of $\Phi^{c^k}_{f^k}$ that any minimizer of $\mathop{\rm arg\,min}\limits_{G(u)=\xi}\Phi_{f^k}'(u)$ is a multiple of one $f^{k+1}\in\mathop{\rm arg\,min}\limits_{G(u)\le1} \Phi_{f^k}^{c^k}(u)$,
so let us look at $\nu f^{k+1}$ with $G(f^{k+1})=1$.
%We have $\Phi'_{f^k}(0)=0\ge\Phi'_{f^k}(f^k)=c^k(1-p)$ and with $\Phi_{f^k}^{c^k}(f^k)=-c^kp$ 
We get from the homogeneity of $\Phi_{f^k}^{c^k}$ and $G$ for $\nu>0$ that
\[ \frac{\partial}{\partial \nu} \big(\Phi'_{f^k}(\nu f^{k+1})\big)%=\frac{\partial}{\partial \nu} \big( \nu\Phi_{f^k}^{c^k}(f^{k+1}) + c^k\nu^p\big) 
= \Phi_{f^k}^{c^k}(f^{k+1}) + c^k p \nu^{p-1} \leq c^k p (\nu^{p-1}-1),\]
which is non-positive for $\nu \in (0,1]$ and with $\Phi'_{f^k}(0)=0\ge\Phi'_{f^k}(f^k)=c^k(1-p)$ it follows that a minimum is attained at $\nu\ge1$.
If $p>1,c^k>0$ then the global optimum of $\Phi'_{f^k}$ exists and by the previous arguments is attained at multiples of $f^{k+1}\in\mathop{\rm arg\,min}\limits_{G(u)\le1} \Phi_{f^k}^{c^k}(u)$.
If $f^k\in\mathop{\rm arg\,min}\limits_{G(u)\le1} \Phi_{f^k}^{c^k}(u)$ then also the global optimum of $\Phi'_{f^k}$ exists
and the claim follows since $\nu=1$ is a minimizer of $\Phi'_{f^k}(\nu f^k)=-\nu c^kp+\nu^p c^k$.
\qed\end{proof}

Note that $G(f^{k})=1$ is no restriction since we get from the one-homogeneity of $\Phi_{f^k}^{c^k}$ that $G(f^k)=1$ for all $k$.
The following lemma verifies the intuition that the strength of the proximal term of RatioDCA-prox %we will show that we can with the parameters $c^k$
controls in some sense how near successive iterates are.

\begin{lemma}
Let
$f_1^{k+1}\in\mathop{\rm arg\,min}\limits_{G(u)\le1}\Phi_{f^k}^{c^k}(u),$ and 
$f_2^{k+1}\in\mathop{\rm arg\,min}\limits_{G(u)\le1}\Phi_{f^k}^{d^k}(u)$.\\
If $c^k\le d^k$ then $\left\langle{f^{k+1}_1,g(f^k)}\right\rangle\le\left\langle{f_2^{k+1},g(f^k)}\right\rangle$.
\end{lemma}
\begin{proof}\smartqed
This follows from
\begin{equation*}
\begin{aligned}
\Phi_{f^k}^{d^k}(f_2^{k+1})&\le\Phi_{f^k}^{d^k}(f_1^{k+1})=\Phi_{f^k}^{c^k}(f_1^{k+1})+(c^k-d^k)\left\langle f_1^{k+1},g(f^k) \right\rangle\\
&\le\Phi_{f^k}^{c^k}(f_2^{k+1})+(c^k-d^k)\left\langle f_1^{k+1},g(f^k) \right\rangle\\
&=\Phi_{f^k}^{d^k}(f_2^{k+1})+(d^k-c^k)\left\langle f_2^{k+1},g(f^k) \right\rangle+(c^k-d^k)\left\langle f_1^{k+1},g(f^k) \right\rangle.\;\qed
\end{aligned}
\end{equation*}
\end{proof}
\begin{remark}
\label{Jost_Setzer_Hein:re:nonexact}
As all proofs can be split up into the individual steps we may choose different functions $G$ in every step of the algorithm.
Moreover, it will not be necessary that $f^{k+1}$ is an exact minimizer of the inner problem, but we will only use that
$\Phi_{f^k}^{c^k}(f^{k+1})<\Phi_{f^k}^{c^k}(f^k)$.
\end{remark}
\subsection{Special cases}
It is easy to see that we get for $c^k=0$ and $G=\|\cdot\|_2$ the RatioDCA \cite{Jost_Setzer_Hein:HeiSet2011} as a special case of the RatioDCA-prox. Moreover, Lemma \ref{Jost_Setzer_Hein:le:decomposition}
shows that the RatioDCA-prox corresponds to the RatioDCA with a general constraint set for the d.c. decomposition of the ratio $F$ given in \eqref{Jost_Setzer_Hein:eq:dc-F}.

If we apply RatioDCA-prox to the ratio cut problem, where $\hat{S}(C)=|C||\overline{C}|$, then $R(u)=R_1(u)=\sum_{i,j=1}^n w_{ij}|u_i-u_j|$ and \cite{Jost_Setzer_Hein:BLUB12} chose $S(u)=S_1(u)=\left\|{u-\mathop{\rm mean}\nolimits(u)\mathbf{1}}_1\right\|$.
The following lemma shows that for a particular choice of $G$ and $c^k$, RatioDCA-prox and algorithm 1 of \cite{Jost_Setzer_Hein:BLUB12}, which calculates iterates $\tilde{f}^{k+1}$ for $v^k\in\partial S(f^k)$ by
\begin{equation*}
 h^{k+1}=\mathop{\rm arg\,min}\limits_u\left\{\frac{1}{2}\sum_{i,j}w_{ij}|u_i-u_j|+\frac{\lambda^k}{2c}\|u-(\tilde f^k+cv^k)\|_2^2\right\},\; \tilde{f}^{k+1}=h^{k+1}/\left\|{h^{k+1}}_2\right\|,
\end{equation*}
 produce the same sequence if given the same initialization.
\begin{lemma}
If $f^0=\tilde f^0$, $\mathop{\rm mean}\nolimits(f^0)=0$, $c>0$ and one uses the same subgradients in each step then, for the sequence $\tilde f^k$ produced by algorithm 1 of \cite{Jost_Setzer_Hein:BLUB12} and $f^k$ 
produced by RatioDCA-prox with $c^k=\frac{\lambda^k}{2c}$ and $G(u)=\|u\|_2^2$, we have $\tilde f^k=f^k$ for all $k$.
\end{lemma}
\begin{proof}\smartqed
If $f^k=\tilde f^k$ and we choose $v^k:=s_1(f^k)=s_1(\tilde f^k)=(\mathbb{I}-\frac{1}{n}\mathbf{1}\mathbf{1}^T)\mathop{\rm sign}\limits(f^k-\mathop{\rm mean}\nolimits(f^k))$. For RatioDCA-prox we get $f^{k+1}$ by
\begin{equation*}
%\label{Jost_Setzer_Hein:eq:proxstepk}
f^{k+1}=\mathop{\rm arg\,min}\limits_{\|u\|^2_2\le1}\Phi_{f^k}^{c^k}(u)
\end{equation*}
and for the algorithm 1 of \cite{Jost_Setzer_Hein:BLUB12}
\begin{equation*}
\begin{aligned}
h^{k+1}&=\mathop{\rm arg\,min}\limits_u\left\{R(u)+\frac{\lambda^k}{2c}\left(\|u\|_2^2-2\left\langle{u,f^k}\right\rangle-2\left\langle{u,cv^k}\right\rangle\right)\right\}\\
&=\mathop{\rm arg\,min}\limits_u\left\{\Phi_{f^k}^{c^k}(u)+\frac{\lambda^k}{2c}\|u\|_2^2\right\}\\
\end{aligned}
\end{equation*}
Finally, $\tilde f^{k+1}=h^{k+1}/\left\|{h^{k+1}}_2\right\|$ and application of Lemma \ref{Jost_Setzer_Hein:le:decomposition} then shows that $\tilde f^{k+1}=f^{k+1}$. 
%Note that due to the strict convexity of $\|.\|_2^2$ the minimizers are unique.
As  $\|.\|_2^2$ is strictly convex, the minimizers are unique.
\qed\end{proof}
Analogously, the algorithm presented in \cite{Jost_Setzer_Hein:BLUB2012} is a special case of RatioDCA-prox applied to the
ratio cheeger cut where $R(u)=R_1(u)=\sum_{i,j}w_{ij}|u_i-u_j|$ and $S(u)=S_1(u)=\sum_i|u_i-\mathop{\rm median}\nolimits(u)|$. 
%In that case one has to choose the same subgradients in each step as they are not unique.

\subsection{Monotony and convergence}
In this section we show that the sequence $F(f^k)$ produced by RatioDCA-prox is monotonically decreasing similar to the RatioDCA of \cite{Jost_Setzer_Hein:HeiSet2011} and, additionally, we can show a convergence property, which generalizes the results of \cite{Jost_Setzer_Hein:BLUB2012,Jost_Setzer_Hein:BLUB12}.

\begin{proposition}
\label{Jost_Setzer_Hein:prop:monotonousfalling}
For every nonnegative sequence $c^k$ any sequence $f^k$ produced by RatioDCA-prox satisfies $F(f^{k+1})<F(f^k)$
for all $k\ge0$ or the sequence terminates. %, that is $F(f^{k+1})=F(f^k)$ and one can choose $f^{k+1}=f^k$. 
Moreover, we get that ${c^k\left\langle{f^{k+1}-f^k,g(f^k)}\right\rangle\to 0}$.
\end{proposition}

\begin{proof}\smartqed
If the sequence does not terminate then $\Phi_{f^k}^{c^k}(f^{k+1})<\Phi_{f^k}^{c^k}(f^k)$ and it follows
\begin{eqnarray*}
R(f^{k+1})-\lambda^kS(f^{k+1})-c^k\left\langle f^{k+1},g(f^k)\right\rangle\le\Phi^{c^k}_{f^k}(f^{k+1})<\Phi^{c^k}_{f^k}(f^{k})=-c^k\left\langle{f^k,g(f^k)}\right\rangle,
\end{eqnarray*}
where we used that for any one-homogeneous convex function $A$ we have for all $f,g\in \mathbb R^V$ and all $a\in\partial A(g)$
\begin{equation*}
 A(f)\ge A(g)+\left\langle f-g,a\right\rangle=\left\langle f,a\right\rangle.
\end{equation*}
Adding $c^k\left\langle f^{k+1},g(f^k)\right\rangle$ gives
\begin{equation}
\label{Jost_Setzer_Hein:eq:monotony}
R(f^{k+1})-\lambda^kS(f^{k+1})<c^k\left\langle f^{k+1},g(f^k)\right\rangle-c^k\left\langle{f^k,g(f^k)}\right\rangle\le0
\end{equation}
where we used that since $G$ is convex
\begin{equation*}
\left\langle f^{k+1},g(f^k)\right\rangle-\left\langle{f^k,g(f^k)}\right\rangle\le G(f^{k+1})-G(f^k)=0.
\end{equation*}
Dividing \eqref{Jost_Setzer_Hein:eq:monotony} by $S(f^{k+1})$ gives $F(f^{k+1})<F(f^k)$.
As the sequence $F(f^{k})$ is bounded from below and monotonically decreasing and thus converging and $S(f^{k+1})$
is bounded on the constraint set, we get the convergence result from
\begin{equation*}
\lambda^{k+1}S(f^{k+1})-\lambda^kS(f^{k+1})\le c^k\left\langle{f^{k+1}-f^k,g(f^k)}\right\rangle\le0.
\end{equation*}
\qed\end{proof}
If we choose $G(u)=\frac{1}{2}\|u\|_2^2$ we get $g(f^k)=f^k$ and if $c^k$ is bounded from below $\|f^{k+1}-f^k\|_2\to0$ as in the case of \cite{Jost_Setzer_Hein:BLUB2012,Jost_Setzer_Hein:BLUB12} but we can show, that this convergence holds for any
strictly convex function $G$. 
\begin{proposition}\label{Jost_Setzer_Hein:prop:conv}
If $G$ is strictly convex and $c^k\geq\gamma>0$ for all $k$, then any sequence $f^k$ produced by RatioDCA-prox fulfills $\|f^{k+1}-f^k\|_2\to0$.
\end{proposition}
\begin{proof}\smartqed
As in the proof of Proposition \ref{Jost_Setzer_Hein:prop:monotonousfalling}, we have $\left\langle{g(f^k),f^{k+1}-f^k}\right\rangle\le0$
and $G(f^{k+1})=G(f^k)=1$. Suppose $f^{k+1}\in G_{\epsilon}:=\{u\vert G(u)=1,\|u-f^k\|\ge\epsilon\}$.
If $\left\langle{g(f^k),f^{k+1}-f^k}\right\rangle=0$, then the first order condition yields for $0<t<1$
\begin{equation*}
 G( f^k + t(f^{k+1}-f^k)) \ge G(f^k) + \left\langle{g(f^k), t(f^{k+1}-f^k)}\right\rangle = G(f^k)=1,
\end{equation*}
which is a contradiction to the strict convexity of $G$ as for $0<t<1$, 
\[  G( f^k + t(f^{k+1}-f^k)) < (1-t)G(f^k) + t G(f^{k+1})=1.\]
Thus with the compactness of $G_{\epsilon}$ we get
\begin{equation*}
\left\langle{g(f^k),f^{k+1}-f^k}\right\rangle\le\mathop{\rm max}\nolimits_{u\in G_{\epsilon}}\left\langle{g(f^k),u-f^k}\right\rangle=:\delta<0.
\end{equation*}
However, with $c^k\geq\gamma >0$ for all $k$ this contradicts for $k$ large enough the result $\left\langle{f^{k+1}-f^k,g(f^k)}\right\rangle \rightarrow 0$ as $k\rightarrow \infty$
of Proposition \ref{Jost_Setzer_Hein:prop:monotonousfalling}. Thus under the stated conditions $\|f^{k+1}-f^k\|_2 \rightarrow 0$ as $k\rightarrow \infty$.
\qed\end{proof}
While the previous result does not establish convergence of the sequence, it establishes that the set of accumulation points has to be connected.

As we are interested in minimizing the ratio $F$ we want to find vectors $f$ with $S(f)\neq0$ 
\begin{lemma}\label{Jost_Setzer_Hein:le:feasible2}
If $S(f^0)\neq0$ then every vector in the sequence $f^k$ produced by RatioDCA-prox fulfills $S(f^k)\neq0$.
\end{lemma}
\begin{proof}\smartqed
As $R_2$ and $S_1$ are one-homogeneous and $G(f^k)=1$, %from the proof of proposition \ref{Jost_Setzer_Hein:prop:monotonousfalling}
we have for any vector $h$ with
$S(h)=0$ and $G(h)=1$,
\begin{equation*}
\begin{aligned}
\Phi^{c^k}_{f^k}(h)&\ge R_1(h)-R_2(h)+\lambda^k(S_2(h)-S_1(h))-c^k\left\langle{h,g(f^k)}\right\rangle\\
&\ge R(h)-c^k\left\langle{f^k,g(f^k)}\right\rangle\ge-c^k\left\langle{f^k,g(f^k)}\right\rangle=\Phi^{c^k}_{f^k}(f^k)
\end{aligned}
\end{equation*}
where we have used that $\left\langle{g(f^k),h}\right\rangle\leq G(h)-G(f^k)+\left\langle{f^k,g(f^k)}\right\rangle=\left\langle{f^k,g(f^k)}\right\rangle$.
Further, if $f^k$ is a minimizer then the algorithm terminates.
\qed\end{proof}
\subsection{Choice of the constraint set and the proximal term}
While the iterates $f^{k}$ and thus the final result of RatioDCA and RatioDCA-prox differ in general, the following lemma shows that termination of RatioDCA
implies termination of RatioDCA-prox and under some conditions also the reverse implication holds true. Thus switching from RatioDCA to RatioDCA-prox at termination does
not allow to get further descent.
\begin{lemma}
\label{Jost_Setzer_Hein:le:equivprox}
Let $f^k_2$, $\|f^k_2\|_2=1$, $f^k_1=\frac{f^k_2}{G(f^k_2)^{\frac{1}{p}}}$, $c^k\ge0$, $s_1(f^k_2)=s_1(f^k_1), r_2(f^k_2)=r_2(f^k_1)$ as in the algorithm RatioDCA-prox and
\begin{equation*}
  \Omega_1 = \mathop{\rm arg\,min}\limits_{G(u)\le1} \Phi_{f^k_1}^{c^k}(u), \quad \textrm{ and } \quad \Omega_2 = \mathop{\rm arg\,min}\limits_{\|u\|_2\le1} \Phi^0_{f^k_2}(u).
\end{equation*}
Then the following implications hold:
\begin{enumerate}
\item If $f^k_2 \in \Omega_2$ then $f^k_1 \in \Omega_1$.
\item If $f_1^{k} \in \Omega_1$ and either $\partial G(f^k_1)=\{g(f^k_1)\}$ or $c^k=0$ then $f_2^{k} \in \Omega_2$.
\end{enumerate}
\end{lemma}

\begin{proof}\smartqed
If $f_2^{k} \in \Omega_2$ then $\Phi^0_{f^k_2}(f^k_2)=0$.
As $\Phi^0_{f^k_2}$ is one-homogeneous, $f^k_2$ is also a global minimizer and thus for all $u \in \mathbb{R}^V$ with $G(u)\leq 1$, 
$\Phi^{c^k}_{f^k_1}(u)=\Phi^0_{f^k_1}(u)-c^k\left\langle{g(f^k_1),u}\right\rangle\geq -c^k \left\langle{g(f^k_1),u}\right\rangle\geq -c^k p$.
As $\left\langle{g(f^k_1),f^k_1}\right\rangle=p$, $f_1^k$ is minimizer which proves the first part.\\
On the other hand if
\begin{equation*}
%\label{Jost_Setzer_Hein:fix}
f^k_1\in\mathop{\rm arg\,min}\limits_{G(u)\le1} \Phi_{f^k_1}^{c^k}(u),
\end{equation*}
then by Lemma \ref{Jost_Setzer_Hein:le:decomposition} also
\begin{equation*}
%\label{Jost_Setzer_Hein:nu}
f^k_1\in\mathop{\rm arg\,min}\limits_u\left\{\Phi_{f^k_1}^{c^k}(u)+c^kG(u)\right\}.
\end{equation*}
$f^k_1$ being a global minimizer implies
\begin{equation*}
\begin{aligned}
0&\in \partial \left(\Phi_{f^k_1}^{c^k}+c^kG\right)(f^k_1)=\partial\Phi^0_{f^k_1}(f^k_1)-c^kg(f^k_1)+c^k\partial G(f^k_1)=\partial\Phi^0_{f^k_1}(f^k_1),
\end{aligned}
\end{equation*}
where we used that by assumption $c^k(g(f_1^k)-\partial G(f_1^k))=0$.
Thus $f^k_1$ is also a minimizer of $\Phi^0_{f^k_1}$ and the result follows with 
$\Phi^0_{f^k_1}(f^k_1)=\Phi^0_{f^k_1}\left(f^k_2\right)=0$ and $\Phi^0_{f^k_1}=\Phi^0_{f^k_2}$.
\qed\end{proof}
\subsection{Nonlinear eigenproblems}
\index{eigenproblem!nonlinear,critical point}
The sequence $F(f^k)$ is not only monotonically decreasing but we also show now that the sequence $f^k$ converges to a generalized nonlinear
eigenvector as introduced in \cite{Jost_Setzer_Hein:HeiBue2010}.

\begin{theorem}
\label{Jost_Setzer_Hein:th:accpoints}
Each cluster point $f^*$ of the sequence $f^k$ produced by RatioDCA-prox fulfills 
for a $c^*$ and with $\lambda^* = \frac{R(f^*)}{S(f^*)} \in \left[ 0, F(f^0)\right]$
\[ 0 \in \partial \big(R_1(f^*)+c^*G(f^*)\big) - \partial \big(R_2(f^*)+c^*G(f^*)\big)
 - \lambda^* \big(\partial S_1(f^*)- \partial S_2(f^*)\big).\]
If for every $f$ with $G(f)=1$ the subdifferential $\partial G(f)$ is unique or $c^k=0$ for all $k$,
then $f^*$ is an eigenvector with eigenvalue $\lambda^*$ in the sense that it fulfills 
\begin{equation}
  \label{Jost_Setzer_Hein:eq:nleigenvector} 
  0 \in \partial R_1(f^*) - \partial R_2(f^*) - \lambda^* \big(\partial S_1(f^*)- \partial S_2(f^*)\big).
\end{equation}
\end{theorem}

\begin{proof}\smartqed
By Proposition \ref{Jost_Setzer_Hein:prop:monotonousfalling} the sequence $F(f^k)$ is monotonically decreasing. 
By assumption $S=S_1-S_2$ and $R=R_1-R_2$ are nonnegative and hence $F$ is bounded below by zero. 
Thus we have convergence towards a limit
\[
	\lambda^* = \lim_{k \rightarrow \infty} F(f^k) \ .
\]
Note that $f^k$ is contained in a compact set, which implies that there exists a subsequence $f^{k_j}$ converging to some element $f^*$. As the sequence $F(f^{k_j})$ is a subsequence of a convergent sequence, it has to converge towards the same limit, hence also
\[
	\lim_{j \rightarrow \infty} F(f^{k_j}) = \lambda^* \ .
\]
Assume now that for all $c$ $\mathop{\rm min}\nolimits_{G(u)\le1} \Phi^c_{f^*}(u)<\Phi^c_{f^*}(f^*)$.
Then by Proposition \ref{Jost_Setzer_Hein:prop:monotonousfalling}, any vector $f^{(c)}  \in \mathop{\rm arg\,min}\limits_{G(u)\le1} \Phi^c_{f^*}(u)$ satisfies
	\[
		F(f^{(c)}) < \lambda^* = F(f^*)\ ,
		\]
which is a contradiction to the fact that the sequence $F(f^k)$ has converged to $\lambda^*$.
Thus there exists $c^*$ such that  %$c^*:=\inf_kc^k\ge0$
$f^*\in\mathop{\rm arg\,min}\limits_{G(u)\le1}\left\{\Phi_{f^*}^{c^*}(u)\right\}$
and by Lemma \ref{Jost_Setzer_Hein:le:decomposition} then $f^*\in\mathop{\rm arg\,min}\limits_u\left\{\Phi_{f^*}^{c^*}(u)+c^*G(u)\right\}$
and we get
\begin{equation*}
\begin{aligned}
0&\in \partial R_1(f^*)-r_2(f^*)+\lambda^*\left(\partial S_2(f^*)-s_1(f^*)\right)-c^*g(f^*)+c^*\partial G(f^*).
\end{aligned}
\end{equation*}
If $c^k=0$ for all $k$ then we only need to look at $c^*=0$.
In this case or if we get from $G(f^*)=1$ that $\partial G(f^*)=\{g(f^*)\}$ it follows that
\begin{equation*}
0\in\partial R_1(f^*)-r_2(f^*)+\lambda^*\left(\partial S_2(f^*)-s_1(f^*)\right)
\end{equation*}
which then implies that $f^*$ is an eigenvector of $F$ with eigenvalue $\lambda^*$. 
\qed\end{proof}
\begin{remark}
 \eqref{Jost_Setzer_Hein:eq:nleigenvector} is a necessary condition for $f^*$ being a critical point of $F$. If $R_2, S_1$ are continuously differentiable at $f^*$, it is also sufficient.
 The necessity of \eqref{Jost_Setzer_Hein:eq:nleigenvector} follows from \cite[Proposition 2.3.14]{Jost_Setzer_Hein:Cla83}.
 If $R_2,S_1$ are continuously differentiable at $f^*$ then we get from \cite[Propositions 2.3.6 and 2.3.14]{Jost_Setzer_Hein:Cla83} that $0\in\partial F(f^*)$ and $f^*$ is a critical point of $F$.
\end{remark}

%%%%%%%%%%%%%%%%%%%%%%%%%%%%%%%%%%%%%%%%%%%%%%%%%%%%%%%%%%%%%%%%%%%%%%%%%%%%%%%

\section{The RatioDCA-prox for Ratios of Lovasz Extensions - Application to Balanced Graph Cuts}
\label{Jost_Setzer_Hein:sec:algforcuts}
\index{graph cuts!balanced}
A large class of combinatorial problems \cite{Jost_Setzer_Hein:HeiSet2011,Jost_Setzer_Hein:BueEtAl2013} allows for an exact continuous relaxation which
results in a minimization problem of a non-negative ratio of Lovasz extensions as introduced in Section \ref{Jost_Setzer_Hein:sec:intro}.
In this paper, we restrict ourselves to balanced graph cuts even though most statements can be immediately generalized to
the class of problems considered in \cite{Jost_Setzer_Hein:BueEtAl2013}.

We first collect some important properties of Lovasz extensions before we prove stronger results for the RatioDCA-prox
when applied to minimize a non-negative ratio of Lovasz extensions.

\subsection{Properties of the Lovasz extension}
\index{Lovasz extension}
The following lemma is a reformulation of \cite[Proposition 4.2(c)]{Jost_Setzer_Hein:Bac2013} for our purposes:

\begin{lemma}\label{Jost_Setzer_Hein:le:lovaszindicators}
Let $\hat S$ be a submodular function with $\hat S(\emptyset)=\hat{S}(V)=0$. If $S$ is the Lovasz extension of $\hat S$ then
\begin{equation*}
\left\langle{\partial S(f),\mathbf{1}_{C_i}}\right\rangle=S(\mathbf{1}_{C_i})=\hat S(C_i)
\end{equation*}
for all sets $C_i=\{j\in V|f_j>f_i\}$.
\end{lemma}
\begin{proof}\smartqed
Let wlog $f$ be in increasing order $f_1\le f_2\le\dots\le f_n$. With $f=\sum_{i=1}^{n-1}\mathbf{1}_{C_i}(f_{i+1}-f_i)+\mathbf{1}_V\cdot f_1$ we get
\begin{equation*}
\begin{aligned}
\sum_{i=1}^n\hat S(C_i)(f_{i+1}-f_i)&=S(f)=\left\langle{\partial S(f),f}\right\rangle=\sum_{i=1}^{n-1}\left\langle{\partial S(f),\mathbf{1}_{C_i}}\right\rangle(f_{i+1}-f_i).
\end{aligned}
\end{equation*}
Since $\hat S$ is submodular $S$ is convex and thus $\left\langle{\partial S(f),\mathbf{1}_{C_i}}\right\rangle\le S(\mathbf{1}_{C_i})=\hat S(C_i)$, but because $f_{i+1}-f_i\ge0$ this holds with equality in all cases.
\qed\end{proof}
More generally this also holds if $\hat S$ is not submodular:

\begin{lemma}\label{Jost_Setzer_Hein:le:lovaszindicators2}
Let $\hat S$ be a set function with $\hat S(\emptyset)=\hat{S}(V)=0$. If $S$ is the Lovasz extension of $\hat S$ then
\begin{equation*}
\left\langle{\partial S(f),\mathbf{1}_{C_i}}\right\rangle=\hat S(C_i)
\end{equation*}
for all sets $C_i=\{j\in V|f_j>f_i\}$. %($\partial S(f)$ can be empty!).
\end{lemma}
\begin{proof}\smartqed
$\hat S$ can be written as the difference of two submodular set functions $\hat S=\hat S_1-\hat S_2$ and the Lovasz extension $S$ of $\hat S$ is the difference of the corresponding Lovasz extensions $S_1$ and $S_2$.
We get $\partial S(f)\subseteq\partial S_1(f)-\partial S_2(f)$ \cite[Propositions 2.3.1 and 2.3.3]{Jost_Setzer_Hein:Cla83} and both $S_1$ and $S_2$ fulfill the conditions of Lemma \ref{Jost_Setzer_Hein:le:lovaszindicators}. Thus
\begin{equation*}
\begin{aligned}
\left\langle{\partial S(f),\mathbf{1}_{C_i}}\right\rangle&\subseteq\left\langle{\partial S_1(f)-\partial S_2(f),\mathbf{1}_{C_i}}\right\rangle=\left\langle{\partial S_1(f),\mathbf{1}_{C_i}}\right\rangle-\left\langle{\partial S_2(f),\mathbf{1}_{C_i}}\right\rangle\\
&=S_1(\mathbf{1}_{C_i})-S_2(\mathbf{1}_{C_i})=S(\mathbf{1}_{C_i})
\end{aligned}
\end{equation*}
and the claim follows since $\partial S(f)$ is nonempty \cite[Proposition 2.1.2]{Jost_Setzer_Hein:Cla83}.
\qed\end{proof}
Also Lovasz extensions are maximal in the considered class of functions:

\begin{lemma}
Let $\hat S$ be a symmetric set function with $\hat{S}(\emptyset)=0$, $S_L$ its Lovasz extension and $S$ any extension fulfilling the properties of Theorem \ref{Jost_Setzer_Hein:th:sets}, that is
$S$ is one-homogeneous, even, convex and $S(f+\alpha \mathbf{1})=S(f)$ for all $f \in \mathbb{R}^V$, $\alpha \in \mathbb{R}$ 
      and $\hat{S}(A):=S(\mathbf{1}_A)$ for all $A \subset V$.
Then $S_L(f)\ge S(f)$ for all $f\in\mathbb{R}^V$.
\end{lemma}
\begin{proof}\smartqed
By Lemma \ref{Jost_Setzer_Hein:le:lovaszindicators2} and using the convexity and one-homogeneity of $S$ we get
\begin{equation*}
\begin{aligned}
S_L(f)&=\sum_{i=1}^{n-1}\hat S(C_i)(f_{i+1}-f_i)=\sum_{i=1}^{n-1}S(\mathbf{1}_{C_i})(f_{i+1}-f_i)\\
&\ge\sum_{i=1}^{n-1}\left\langle{\partial S(f),\mathbf{1}_{C_i}}\right\rangle(f_{i+1}-f_i)=\left\langle{\partial S(f),f}\right\rangle=S(f)
\end{aligned}
\end{equation*}
\qed\end{proof}
\begin{remark}
By \cite[Lemma 3.1]{Jost_Setzer_Hein:HeiSet2011} any function $S$ fulfilling the properties of the lemma can be rewritten by $S(f)=\sup_{u\in U}\left\langle u,f\right\rangle$ 
where $U\subset\mathbb{R}^n$ is a closed symmetric convex set and $\left\langle u,\mathbf{1}\right\rangle=0$ for all $u\in U$.
The previous lemma implies that for a given set function $\hat{S}(C)$ the set $U$ is maximal for the Lovasz extension $S_L$. In turn this implies that the subdifferential
of $S_L$ is maximal everywhere and thus should be used in the RatioDCA-prox. In \cite{Jost_Setzer_Hein:HeiSet2011, Jost_Setzer_Hein:BLUB12} the authors use for the balancing function $\hat{S}(C)=|C||\overline{C}|$ instead
of the Lovasz extension $S_L(f)=\frac{1}{2}\sum_{i,j=1}^n |f_i-f_j|$ the convex function $S(f)=\left\|{f-\mathop{\rm mean}\nolimits{f}\mathbf{1}}_1\right\|$ which fulfills the properties of the previous lemma.
In Section \ref{Jost_Setzer_Hein:sec:experiments} we show that using the Lovasz extension leads almost always to better balanced graph cuts.
\end{remark}

\subsection{The RatioDCA-prox for balanced graph cuts}
\index{RatioDCA-prox,graph cuts!balanced}
Applied to balanced graph cuts we can show the following ``improvement theorem'' generalizing the result of \cite{Jost_Setzer_Hein:HeiSet2011} for our algorithm. It implies that we can use the result of any other graph partitioning method as initialization and in particular, we can
always improve the result of spectral clustering.
\begin{theorem}\label{Jost_Setzer_Hein:th:improve}
Let $(A,\overline{A})$ be a given partition of $V$ and let $S:V \rightarrow \mathbb{R}_+$ satisfy one of the 
conditions stated in Theorem \ref{Jost_Setzer_Hein:th:sets}. If one uses as initialization of RatioDCA-prox $f^0=\mathbf{1}_A$, 
then either the algorithm terminates after one step or it yields an $f^1$ which after optimal thresholding as in Theorem \ref{Jost_Setzer_Hein:th:sets} gives a partition $(B,\overline{B})$ 
which satisifies
\[ \frac{\mathrm{cut}(B,\overline{B})}{\hat{S}(B)} < \frac{\mathrm{cut}(A,\overline{A})}{\hat{S}(A)}.\]
\end{theorem}

\begin{proof}\smartqed
This follows in the same way from Proposition \ref{Jost_Setzer_Hein:prop:monotonousfalling} as in \cite[Theorem 4.2]{Jost_Setzer_Hein:HeiSet2011}.
% As $f^0=\mathbf{1}_A$ we get with 
% \[ R(\mathbf{1}_A) = \frac{1}{2}\sum_{i,j=1}^n w_{ij} |\mathbf{1}_A(i)-\mathbf{1}_A(j)|=\mathrm{cut}(A,\overline{A}), \textrm{ and } S(\mathbf{1}_A)=\hat{S}(A),\]
% for the initial value of the ratio $F(f^0)=\mathrm{cut}(A,\overline{A})/\hat{S}(A)$. Proposition \ref{Jost_Setzer_Hein:le:monotony_funct} and proposition \ref{Jost_Setzer_Hein:prop:monotonousfalling} now state
% that either the algorithm produces a $f^1$ with $F(f^1) < F(f^0)$ or it terminates. If the second case happens we
% are done, let us consider the first case. By Theorem \ref{Jost_Setzer_Hein:th:sets} there exists a set $C^*$ in the set of sets $C_i=\{ j \in V \,|\, f^1_j>f^1_i\}$, $i=1,\ldots,n-1$ such that $F(f^1)\geq F(\mathbf{1}_{C^*})$. In total we get that there exists a partition $(C^*,\overline{C^*})$ such that
% \[ \mathrm{cut}(C^*,\overline{C^*})/\hat{S}(C^*) = F(\mathbf{1}_{C^*}) \leq F(f^1) < F(f^0)=F(\mathbf{1}_A) = \mathrm{cut}(A,\overline{A})/\hat{S}(A).\]
\qed\end{proof}

In the case that we have Lovasz extensions we can show that accumulation points are directly related to the optimal sets:

\begin{theorem}
If $R_2$ and $S_1$ are Lovasz-extensions of the corresponding set functions then every accumulation point $f^*$ of RatioDCA-prox with $c^k=0$ fulfills $F(f^*)=F(\mathbf{1}_{C^*})$ 
where $C^*$ is the set we get from optimal thresholding of $f^*$.
If also $R_1$ and $S_2$ are the Lovasz-extensions then $f^*=\sum_{i=1}^{m}\alpha_i\mathbf{1}_{C_i}+b\mathbf{1}_V$ with $\alpha_i>0$, $C_i=\{j \in V\,|\,f^*_j>f^*_i\}$, $b \in \mathbb{R}$,
and 
\begin{equation*}
\frac{\hat R(C_i)}{\hat S(C_i)}=\lambda^*=\frac{R(f^*)}{S(f^*)}, \quad i=1,\ldots,m.
\end{equation*}
If $\lambda^*$ is only attained for one set $C^*$ then $f^*=\mathbf{1}_{C^*}$ is the only accumulation point.
\end{theorem}
\begin{proof}\smartqed
In the proof of Theorem \ref{Jost_Setzer_Hein:th:accpoints} it has been shown that from $f^*$ no further descent is possible.
Assume $F(f^*)>F(\mathbf{1}_{C^*})$. Then
\begin{equation*}
\begin{aligned}
\Phi^0_{f^*}(\mathbf{1}_{C^*})&=R_1(\mathbf{1}_{C^*})-\left\langle{r_2(f^*),\mathbf{1}_{C^*}}\right\rangle+\lambda^*(S_2(\mathbf{1}_{C^*})-\left\langle{s_1(f^*),\mathbf{1}_{C^*}}\right\rangle=R(\mathbf{1}_{C^*})-\lambda^*S(\mathbf{1}_{C^*})\\
&<R(\mathbf{1}_{C^*})-F(\mathbf{1}_{C^*})S(\mathbf{1}_{C^*})=0=\Phi^0_{f^*}(f^*)
\end{aligned}
\end{equation*}
which leads to a contradiction. Thus the first claim follows from Theorem \ref{Jost_Setzer_Hein:th:sets}.
If also $R_1$ and $S_2$ are the Lovasz-extensions then for $f^*=\sum_{i=1}^{n-1}\alpha_i\mathbf{1}_{C_i}+\mathbf{1}_V\cdot \mathop{\rm min}\nolimits_jf^*_j$
we get by Lemma \ref{Jost_Setzer_Hein:le:lovaszindicators} and the definition of the Lovasz extension that %\from $R_1(f)=\sum_{i=1}^{n-1}R_1(\mathbf{1}_{C_i})(f_{i+1}-f_i)$ that
\begin{equation*}
0=\Phi^0_{f^*}(f^*)=\sum_{i=1}^n\alpha_i\Phi^0_{f^*}(\mathbf{1}_{C_i})
\end{equation*}
and if for one $\alpha_i>0$ we have $\frac{\hat R(C_i)}{\hat S(C_i)}>\lambda^*$ then $\Phi^0_{f^*}(\mathbf{1}_{C_i})>0$ and we get $\Phi^0_{f^*}(\mathbf{1}_{C^*})<0=\Phi^0_{f^*}(f^*)$ which again is a contradiction.
\qed\end{proof}
\begin{remark}
By Lemma \ref{Jost_Setzer_Hein:le:equivprox} this also holds for $c^k>0$ if $G$ is differentiable at the boundary.
\end{remark}
If we have Lovasz extensions we can also use the reduced version of the RatioDCA-prox with $c^k=0$ 
to guarantee termination. We are thus in the striking situation that in general we can guarantee stronger convergence properties
if $c^k\geq \gamma>0$ for all $k$ by Proposition \ref{Jost_Setzer_Hein:prop:conv} but an even stronger property such as finite convergence
can only be proven when $c^k=0$.
\begin{theorem}
Let $c^k=0$ and $S_1,R_2$ be Lovasz extensions in the RatioDCA-prox. Further, let $C_k^*$ be the set obtained by optimal thresholding of $f^k$.
If in step 5 of RatioDCA-prox we choose, $\lambda^k=F(\mathbf{1}_{C^*_k})$, and in step 4 choose $f^{k+1}=\mathbf{1}^*=\frac{\mathbf{1}_{C^*_k}}{G(\mathbf{1}_{C_k^*})^{\frac{1}{p}}}$ if $\mathbf{1}^*\in\mathop{\rm arg\,min}\limits_{G(u)\le1}\Phi_{f^k}^{c^k}$,
then the RatioDCA-prox terminates in finitely many steps.
\end{theorem}
\begin{proof}\smartqed
%Let $\mathbf{1}^*=\frac{\mathbf{1}_{C^*_k}}{G(\mathbf{1}_{C_k^*})^{\frac{1}{p}}}$ with $G(\mathbf{1}^*)=1$, $F(\mathbf{1}^*)=F(\mathbf{1}_{C_k^*})$.
With $c^k=0$ and using Lemma \ref{Jost_Setzer_Hein:le:lovaszindicators2} and as $R_1,S_2$ are convex and one-homogeneous, we get 
\begin{align*}
&\;R(f^{k+1})-F(\mathbf{1}_{C^*_k})S(f^{k+1})\le\Phi^{c^k}_{f^k}(f^{k+1})\\
\le & \;\Phi^{c^k}_{f^k}(\mathbf{1}^*)=R_1(\mathbf{1}^*)-\left\langle{r_2(f^k),\mathbf{1}^*}\right\rangle + F(\mathbf{1}_{C_k^*})\big(S_2(\mathbf{1}^*)-\left\langle{s_1(f^k),\mathbf{1}^*}\right\rangle\big)\\
=&\; 
R(\mathbf{1}^*)-F(\mathbf{1}_{C_k^*})S(\mathbf{1}^*)=0
\end{align*}
and thus $F(\mathbf{1}_{C_{k+1}^*})\leq F(f^{k+1})\leq F(\mathbf{1}_{C_k^*})$ and equality in the second inequality only holds if $f^{k+1}=\mathbf{1}^*$, but then in the next step we either get strict improvement or the sequence terminates.
As there are only finitely many different cuts, RatioDCA-prox has to terminate in finitely many steps.
\qed\end{proof}
\section{Experiments}
\label{Jost_Setzer_Hein:sec:experiments}
%\subsection{Implementation}

The convex inner problem in Equation \eqref{Jost_Setzer_Hein:eq:innerprob} is solved using the primal dual hybrid gradient method (PDHG)
as in \cite{Jost_Setzer_Hein:HeiSet2011}. In the first iterations the problem is not solved to high accuracy as all results in this paper
only rely on the fact that either the algorithm terminates or
\[ \phi_{f^k}^{c^k}(f^{k+1}) < \phi_{f^k}^{c^k}(f^k).\]

\subsection{Influence of the proximal term}
First, we study the influence of different values of $c^k$ in the RatioDCA-prox algorithm.
We choose $G=\|\cdot\|_2^2$ and choose different values for $c^k$.

We compare the algorithms on the wing graph from \cite{Jost_Setzer_Hein:Wa04} ($62032$ vertices, $243088$ edges) and a graph built from
the two-moons dataset ($2000$ vertices, $33466$ edges) as described in \cite{Jost_Setzer_Hein:BueHei2009}.

\begin{table}
\caption{Displayed are the averages of all, the 10 best and the best cuts for different values of $c^k = c \lambda^k$ on wing (top) and two-moons (bottom).}
\vspace{+2mm}
\begin{tabular}{c@{$\;\;$}|@{$\;\;$}c@{$\;\;$}c@{$\;\;$}c@{$\;\;$}c@{$\;\;$}c@{$\;\;$}c@{$\;\;$}c@{$\;\;$}c@{$\;\;$}c@{$\;\;$}c@{$\;\;$}}
\label{Jost_Setzer_Hein:tab:comparingprox} 
graph $\backslash$ $c$ & 0 & 0.1 & 0.25 & 0.5 & 0.75 & 1 & 1.5 & 2 & 3 & 4\\ 
\hline
wing\\
\hline
avg& 2.6683& 2.6624& 2.6765& 2.6643& 2.6602& 2.6595& 2.6566& 2.6565& \bf{2.6548}& 2.6573\\ 
%top 50 avg& 2.6128& 2.6046& 2.6261& 2.6071& 2.6024& 2.5987& 2.5961& 2.5987&  \bf{2.5956}& 2.5972\\ 
top 10 avg& 2.5554& 2.5519& 2.5625& 2.5533& 2.5549& \bf{2.5514}& 2.5523& 2.5605& 2.5555& 2.5523\\ 
best cut& 2.545& \bf{2.5439}& 2.5532& 2.5487& 2.5451& 2.5471& 2.5448& 2.5539& 2.5472& 2.5472\\ 
\hline
two-moons\\
\hline
avg& 2.4872& 2.4855& 2.5017& 2.5158& \bf{2.4569}& 2.4851& 2.4848& 2.7868& 3.028& 2.929\\ 
%top 50 avg & 2.4514& \bf{2.4509}& 2.4517& 2.4515& 2.4515& 2.4514& 2.4515& 2.4515& 2.4516& 2.4515\\ 
top 10 avg & \bf{2.448}& 2.4485& 2.4481& 2.4487& 2.4484& 2.4484& 2.4481& 2.4492& 2.4491& 2.4483\\ 
best cut & 2.4447& 2.4473& 2.4472& 2.4461& 2.4457& 2.4476& 2.4465& 2.4482& 2.4478& \bf{2.4441}\\ 
\end{tabular}
\end{table}
In Table \ref{Jost_Setzer_Hein:tab:comparingprox} we have plotted the resulting ratio cheeger cuts ($\mathrm{RCC}$)
of ten different choices of $c^k=c\cdot\lambda^k$ for RatioDCA-prox.
In all cases we use one initialization with the second eigenvector of
the standard graph Laplacian and $99$ initializations with random vectors, which are the same for all algorithms.
As one is interested in the best result and how often this can be achieved, we report the best, average and top10 performance.
%For the choice of the $c^k$ we use $c^k=c\cdot\lambda^k$ where 
%$c$ takes for the nine algorithms the ten values $0,0.1,0.25,0.5,0.75,1,1.5,2,3$, and $4$.
For both graphs there is no clear trend that a particular choice of the proximal term improves
or worsens the results compared to $c^k=0$ which corresponds to the RatioDCA.
%Using our implementation the time consumement of the algorithm increases with larger $c$.
This confirms the reported results of \cite{Jost_Setzer_Hein:BLUB2012} where also no clear difference between $c^k=0$ and the general case has been observed.

\subsection{Comparing the Lovasz extension to other extensions}
In previous work \cite{Jost_Setzer_Hein:HeiSet2011,Jost_Setzer_Hein:BLUB12} on the ratio cut with the balancing function $\hat{S}(C)=|C||\overline{C}|$ not the Lovasz extension
$S_L(f)=\frac{1}{2}\sum_{i,j=1}^n |f_i-f_j|$ has been used but the function $S(f)=\left\|{f-\mathop{\rm mean}\nolimits{f}\mathbf{1}}_1\right\|$. As discussed in Section \ref{Jost_Setzer_Hein:sec:algforcuts},
this should lead to worse performance in the algorithm as the subdifferential of $S_L$ is maximal.
In Table \ref{Jost_Setzer_Hein:tab:lovasz} we compare both extensions with the RatioDCA-prox with $c^k=0$ and $G(u)=\left\|{u}_2\right\|$ %a two-moons graph with variance $0.2$ and five other graphs from
on seven different graphs \cite{Jost_Setzer_Hein:Wa04}. One initialization is done with the second eigenvector of the standard graph laplacian and the same $10$ random initializations are used for both extensions.

\begin{table}
\caption{\label{Jost_Setzer_Hein:tab:lovasz}%Comparing Methods: (1)For how many initializations was method 1 better than method 2? (2)Method 2 better than method 1? (3)Method 2 better than the minimum of method 1? (4)Factor of minima.
For each graph it is shown how many times for the 11 initializations the RatioDCA-prox with the Lovasz extension performs better/equal/worse than the previously used continuous extension
and the ratio of the best solutions of Lovasz vs continuous extension is shown ($<100\%$ means that the Lovasz extension produced a better ratio cut).}
\vspace{+2mm}
\begin{tabular}{c@{$\;\;$}|@{$\;\;$}c@{$\;\;$}|@{$\;\;$}c@{$\;\;$}|@{$\;\;$}c@{$\;\;$}|@{$\;\;$}c@{$\;\;$}|@{$\;\;$}c@{$\;\;$}|@{$\;\;$}c@{$\;\;$}|@{$\;\;$}c}
Graph&two-moons&whitaker3&uk&4elt&fe\_4elt&3elt&crack\\
\hline
%Standard extension better&0 &0 &0&1 &0&1 &1\\
%Lovasz extension better  &11&11&0&10&0&10&10\\
%Lovasz extension better than min&11&11&0&9&0&8&6\\
better/equal/worse & 11/0/0 & 11/0/0 & 0/11/0 & 10/0/1 & 0/11/0 & 10/0/1 & 10/0/1\\
Ratio of best cuts &99.41\% &99.95\% &100\%&99.98\%&100\%&99.97\%&99.83\%
\end{tabular}
\end{table}
While the differences in the best found cut are minor, using the Lovasz extension for the balancing function leads consistently to better results.
%While the best cuts found with the standard extension can often not be improved much further, we can clearly see that the Lovasz extension outperforms the standard extension.
%Moreover, we see for the $\mathrm{RCut}$ that the lovasz extension always produces results very close to the optimal value, but as one can see from table \ref{Jost_Setzer_Hein:tab:comparingprox} this does not seem to hold in general.

\bibliographystyle{spphyswithtitles}
\bibliography{Jost_Setzer_Hein:regul}
\end{document}